\documentclass[11pt,a4paper]{article}
\PassOptionsToPackage{breaklinks}{hyperref}
\usepackage[hyperref]{acl2020}
\usepackage{times}
\usepackage{latexsym}

\usepackage{microtype}

\usepackage{amsmath}
\usepackage{amsthm}
\usepackage{paralist}
\usepackage{enumitem}
\usepackage{amsmath}
\usepackage{booktabs}
\usepackage{multirow}
\usepackage{todo}

\DeclareMathOperator*{\argmax}{argmax}
\DeclareMathOperator{\sgn}{sgn}
\DeclareMathOperator{\R}{R}
\DeclareMathOperator{\rect}{R}

\newcommand{\lbl}[1]{{\bf \tt #1}}
\newcommand{\sS}{\mathcal{S}}

\newcommand{\bz}{\mathbf{z}}

\newtheorem{theorem}{Theorem}

\aclfinalcopy 
\def\aclpaperid{1248} 

\title{Learning Constraints for Structured Prediction Using Rectifier Networks}
\author{Xingyuan~Pan, Maitrey~Mehta, Vivek~Srikumar \\
  School of Computing, University of Utah \\
  \texttt{\{xpan,maitrey,svivek\}@cs.utah.edu 
  }\\}

\date{}

\begin{document}
\maketitle

\begin{abstract}
 Various natural language processing tasks are structured prediction problems where outputs are constructed with multiple interdependent decisions. Past work has shown that domain knowledge, framed as constraints over the output space, can help improve predictive accuracy. However, designing good constraints often relies on domain expertise. In this paper, we study the problem of learning such constraints. We frame the problem as that of training a two-layer rectifier network to identify valid structures or substructures, and show a construction for converting a trained network into a system of linear constraints over the inference variables. Our experiments on several NLP tasks show that the learned constraints can improve the prediction accuracy, especially when the number of training examples is small.
\end{abstract}

\section{Introduction}
In many natural language processing (NLP) tasks, the outputs are structures which can take the form of sequences, trees, or in general, labeled graphs. Predicting such output structures~\citep[e.g.][]{Smith:2011} involves assigning values to multiple interdependent variables. Certain joint assignments may be prohibited by constraints designed by domain experts. As a simple example, in the problem of extracting entities and relations from text, a constraint could disallow the relation ``married to'' between two entities if one of the entity is not a ``person''.
It has been shown that carefully designed constraints can substantially improve  model performance in various  applications~\citep[e.g.,][]{Chang2012,Anzaroot2014}, especially when the number of training examples is limited.

Designing constraints often requires task-specific manual effort. In this paper, we ask the question: \emph{can we use neural network methods to automatically discover constraints from data, and use them to predict structured outputs?} We provide a general framework for discovering constraints in the form of a system of linear inequalities over the output variables in a problem. These constraints can improve an already trained model, or be integrated into the learning process for global training.

A system of linear inequalities represents a bounded or unbounded convex polytope. We observe that such a system can be expressed as a two-layer threshold network, i.e., a network with one hidden layer of linear threshold units and an output layer with a single threshold unit. This two-layer threshold network will predict $1$ or $-1$ depending on whether the system of linear inequalities is satisfied or not. In principle, we could try to train such a threshold network to discover constraints. However, the zero-gradient nature of the threshold activation function prohibits using backpropagation for gradient-based learning.

Instead, in this paper, we show that a construction of a specific two-layer rectifier network represents linear inequality constraints. This network also contains a single linear threshold output unit, but in the hidden layer, it contains rectified linear units (ReLUs). \citet{Pan2016} showed that a two-layer rectifier network constructed in such a way is equivalent to a threshold network, and represents the same set of linear inequalities as the threshold network with far fewer hidden units. 

The linear constraints thus obtained can augment existing models in multiple ways. For example, if a problem is formulated as an integer program~\cite[e.g.,][]{Roth2004,Roth2005a,Riedel2006,Martins2009}, the learned constraints will become additional linear inequalities, which can be used directly. Alternatively, a structure can be constructed using graph search~\cite[e.g.,][]{Collins2004,Daume2009,doppa2014structured,Chang2015,Wiseman2016}, in which case the learned constraints can filter available actions during search-node expansions. Other inference techniques that extend Lagrangian Relaxation~\citep{Komodakis2007,Rush2010,martins2011augmented} can also employ the learned constraints. Essentially, the learned constraints can be combined with various existing models and inference techniques and the framework proposed in this paper can be viewed as a general approach to improve structured prediction.

We report experiments on three NLP tasks to verify the proposed idea.
The first one is an entity and relation extraction task, in which we aim to label the entity candidates and identify relations between them. In this task, we show that the learned constraints can be used while training the model to improve prediction. We also show that the learned constraints in this domain can be interpreted in a way that is comparable to manually designed constraints. 
The second NLP task is to extract citation fields like authors, journals and date from a bibliography entry. We treat it as a sequence labeling problem and show that learned constraints can improve an existing first-order Markov model trained using a structured SVM method~\citep{Tsochantaridis2004}. In the final experiment we consider chunking, i.e., shallow parsing, which is also a sequence labeling task. We train a BiLSTM-CRF model~\cite{huang2015bidirectional} on the training set with different sizes, and we show that learned constraints are particularly helpful when the number of training examples is small.

In summary, the contributions of this paper are:
\begin{enumerate}[nosep]
\item We propose that rectifier networks can be used to represent and learn linear constraints for structured prediction problems.

\item In tasks such as entity and relation extraction, the learned constraints can exactly recover the manually designed constraints, and can be interpreted in a way similar to manually designed constraints.

\item When manually designed constraints are not available, we show via
  experiments that the learned constraints can improve the original model's
  performance, especially when the original model is trained with a small
  dataset.\footnote{The scripts for replaying the experiments are available at \href{https://github.com/utahnlp/learning-constraints}{ https://github.com/utahnlp/learning-constraints}}
\end{enumerate}

\section{Representing Constraints}

In this section, we  formally define structured prediction and constraints. In a structured prediction problem, we are given an input $\mathbf{x}$ belonging to the instance space, such as sentences or images. The goal is to predict an output $\mathbf{y} \in \mathcal{Y}_\mathbf{x}$, where $\mathcal{Y}_\mathbf{x}$ is the set of possible output structures for the input $\mathbf{x}$. The output $\mathbf{y}$ have a predefined structure (e.g., trees, or labeled graphs), and the number of candidate structures in $\mathcal{Y}_\mathbf{x}$ is usually large, i.e., exponential in the input size.

Inference in such problems can be framed as an optimization problem with a linear objective function:
\begin{equation}\label{general-argmax}
\mathbf{y}^* = \argmax_{\mathbf{y}\in\mathcal{Y}_\mathbf{x}} \boldsymbol{\alpha} \cdot \boldsymbol{\phi}(\mathbf{x}, \mathbf{y}),
\end{equation} 
where $\boldsymbol{\phi}(\mathbf{x}, \mathbf{y})$ is a feature vector representation of the input-output pair $\mathbf{(x,y)}$ and $\boldsymbol{\alpha}$ are learned parameters. The feature representation $\boldsymbol{\phi}(\mathbf{x}, \mathbf{y})$ can be designed by hand or learned using neural networks. The feasible set $\mathcal{Y}_\mathbf{x}$ is predefined and known for every $\mathbf{x}$ at both learning and inference stages. The goal of learning is to find the best parameters $\boldsymbol{\alpha}$ (and, also perhaps the features $\boldsymbol{\phi}$ if we are training a neural network) using training data, and the goal of inference is to solve the above argmax problem given parameters $\boldsymbol\alpha$.

In this paper, we seek to learn additional constraints from training examples $\{\mathbf{(x,y)}\}$. Suppose we want to learn $K$ constraints, and the $k^{\text{th}}$ one is some Boolean function\footnote{We use $1$ to indicate {\tt true} and $-1$ to indicate {\tt false}.}: $c_k(\mathbf{x,y})=1$ if $(\mathbf{x, y})$ satisfies the $k^{th}$ constraint, and $c_k(\mathbf{x,y})=-1$ if it does not.
Then, the optimal structure $\mathbf{y}^*$ is the solution to the
following optimization problem:
\begin{align}
&\max_{\mathbf{y}\in\mathcal{Y}_\mathbf{x}} \boldsymbol\alpha \cdot \boldsymbol\phi(\mathbf{x}, \mathbf{y}),\label{constrainted-argmax}\\ 
&\text{subject to}\quad \forall k, c_k(\mathbf{x, y}) = 1.\nonumber
\end{align} 
We will show that such learned constraints aid prediction performance.

\subsection{Constraints as Linear Inequalities}

Boolean functions over inference variables may be expressed as linear inequalities over them~\citep{Roth2004}. In this paper, we represent constraints as linear inequalities over some feature vector $\boldsymbol\psi(\mathbf{x, y})$ of a given input-output pair. The $k^\text{th}$ constraint $c_k$  is equivalent to the linear inequality
\begin{equation}\label{threshold_constraint}
\mathbf{w}_k \cdot \boldsymbol{\psi}(\mathbf{x, y}) + b_k \ge 0,
\end{equation}
whose weights $\mathbf{w}_k$ and bias $b_k$ are learned. A Boolean constraint is, thus, a linear threshold function,
\begin{equation}\label{ltu}
c_k(\mathbf{x,y}) = \sgn \big(\mathbf{w}_k \cdot \boldsymbol{\psi}(\mathbf{x, y}) + b_k\big).
\end{equation}
Here, $\sgn(\cdot)$ is the sign function: $\sgn(x) = 1$ if $x\ge 0$, and $-1$ otherwise.

The feature representations $\boldsymbol\psi(\mathbf{x, y})$ should not be confused with the original features $\boldsymbol\phi(\mathbf{x}, \mathbf{y})$ used in the structured prediction model in Eq.~\eqref{general-argmax} or~\eqref{constrainted-argmax}. Hereafter, we refer to $\boldsymbol\psi(\mathbf{x, y})$ as \emph{constraint features}. Constraint features should be general properties of inputs and outputs, since we want to learn domain-specific constraints over them. They are a design choice, and in our experiments, we will use common NLP features. In general, they could even be learned using a neural network. Given a constraint feature representation $\boldsymbol\psi(\cdot)$, the goal is thus to learn the parameters $\mathbf{w_k}$'s and $b_k$'s for every constraint.

\subsection{Constraints as Threshold Networks}

For an input $\mathbf{x}$, we say the output $\mathbf{y}$ is feasible if it satisfies constraints $c_k$ for all $k=1,\dots,K$. We can define a Boolean variable $z(\mathbf{x,y})$ indicating whether $\mathbf{y}$ is feasible with respect to the input $\mathbf{x}$:
$z(\mathbf{x,y}) = c_1(\mathbf{x,y})\land \dots \land c_K(\mathbf{x,y})$.
That is, $z$ is a conjunction of all the Boolean functions corresponding to each constraint. Since conjunctions are linearly separable, we can rewrite $z(\mathbf{x,y})$ as a linear threshold function:
\begin{equation}\label{lt}
z(\mathbf{x,y}) = \sgn\Big(1-K + \sum_{k=1}^K  c_k(\mathbf{x,y})\Big).
\end{equation}
It is easy to see that $z(\mathbf{x,y})=1$ if, and only if, all $c_k$'s are $1$---precisely the definition of a conjunction. Finally, we can plug Eq.~\eqref{ltu} into Eq.~\eqref{lt}:
\begin{equation}\label{threshold}
z = \sgn\Big(1-K 
+ \sum_{k=1}^K \sgn \big(\mathbf{w}_k \cdot \boldsymbol{\psi}(\mathbf{x, y}) + b_k\big)\Big)
\end{equation}

Observe that Eq.~\eqref{threshold} is exactly a two-layer threshold neural network: $\boldsymbol{\psi}(\mathbf{x, y})$ is the input to the network; the hidden layer contains $K$ linear threshold units with parameters $\mathbf{w}_k$ and $b_k$; the output layer has a single linear threshold unit. This neural network will predict $1$ if the structure $\mathbf{y}$ is feasible with respect to input $\mathbf{x}$, and $-1$ if it is infeasible. In other words, constraints for structured prediction problems can be written as two-layer threshold networks. One possible way to learn constraints is thus to learn the hidden layer parameters $\mathbf{w}_k$ and $b_k$, with fixed output layer parameters. However, the neural network specified in Eq.~\eqref{threshold} is not friendly to gradient-based learning; the $\sgn(\cdot)$ function has zero gradients almost everywhere. To circumvent this, let us explore an alternative way of learning constraints using rectifier networks rather than threshold networks.

\subsection{Constraints as Rectifier Networks}

We saw in the previous section that a system of linear inequalities can be represented as a two-layer threshold network. In this section, we will see a special rectifier network that is equivalent to a system of linear inequalities, and whose parameters can be learned using backpropagation.

Denote the rectifier (ReLU) activation function as $\R(x)= \max(0, x)$. Consider the following two-layer rectifier network:
\begin{equation}\label{relu_network}
z= \sgn\Big(1- \sum_{k=1}^K \R \big(\mathbf{w}_k \cdot \boldsymbol{\psi}(\mathbf{x, y}) + b_k\big)\Big)
\end{equation}
The input to the network is still $\boldsymbol{\psi}(\mathbf{x, y})$. There are $K$ ReLUs in the hidden layer, and one threshold unit in the output layer.
The decision boundary of this rectifier network is specified by a system of linear inequalities. In particular, we have the following theorem \citep[Theorem~1]{Pan2016}:
\begin{theorem}\label{theorem}
Consider a two-layer rectifier network with $K$ hidden ReLUs as in Eq.~\eqref{relu_network}. Define the set $[K]=\{1,2,\dots,K\}$. The network output $z(\mathbf{x,y})=1$ if, and only if, for every subset $\mathcal{S}$ of $[K]$, the following linear inequality holds:
\begin{equation}\label{ineqs}
1-\sum_{k \in \mathcal{S}}\big(\mathbf{w}_k \cdot \boldsymbol{\psi}(\mathbf{x, y}) + b_k\big) \ge 0
\end{equation}
\end{theorem}
The proof of Theorem~\ref{theorem} is given in the supplementary material.

To illustrate the idea, we  show a simple example rectifier network, and convert it to a system of linear inequalities using the theorem. The rectifier network contains two hidden ReLUs ($K=2$): 
\begin{equation*}\label{example_relu}
z = \sgn\Big(1- \R \big(\mathbf{w}_1 \cdot \boldsymbol{\psi} + b_1\big) 
- \R \big(\mathbf{w}_2 \cdot \boldsymbol{\psi} + b_2\big)\Big)
\end{equation*}
Our theorem says that $z=1$ if and only if the following four inequalities hold simultaneously, one per subset of $[K]$:
\begin{align*}
\begin{cases}
1 \ge 0 \\
1 - \big(\mathbf{w}_1 \cdot \boldsymbol{\psi} + b_1\big) \ge 0  \\
1 - \big(\mathbf{w}_2 \cdot \boldsymbol{\psi} + b_2\big) \ge 0 \\
1- \big(\mathbf{w}_1 \cdot \boldsymbol{\psi} + b_1\big)- \big(\mathbf{w}_2 \cdot \boldsymbol{\psi} + b_2\big) \ge 0 
\end{cases}
\end{align*}
The first inequality, $1 \ge 0$, corresponding to the empty subset of $[K]$, trivially holds. The rest are just linear inequalities over $\boldsymbol{\psi}$.

In general, $[K]$ has $2^K$ subsets, and when $\mathcal{S}$ is the empty set, inequality~\eqref{ineqs} is trivially true. The rectifier network in Eq.~\eqref{relu_network} thus predicts $\mathbf{y}$ is a valid structure for $\mathbf{x}$, if a system of $2^K-1$ linear inequalities are satisfied. It is worth mentioning that even though the $2^K-1$ linear inequalities are constructed from a power set of $K$ elements, it does not make them dependent on each other. With general choice of $\mathbf{w}_k$ and $b_k$, these $2^K-1$ inequalities are linearly independent.

This establishes the fact that a two-layer rectifier network of the form of Eq.~\eqref{relu_network} can represent a system of linear inequality constraints for a structured prediction problem via the constraint feature function $\boldsymbol\psi$.

\section{Learning Constraints}

In the previous section, we saw that both threshold and rectifier networks can represent a system of linear inequalities. We can either learn a threshold network (Eq.~\eqref{threshold}) to obtain constraints as in~\eqref{threshold_constraint}, or we can learn a rectifier network  (Eq.~\eqref{relu_network}) to obtain constraints as in~\eqref{ineqs}.  The latter offers two advantages. First, a rectifier network has non-trivial gradients, which facilitates gradient-based learning\footnote{The output threshold unit in the rectifier network will not cause any trouble in practice, because it can be replaced by sigmoid function during training. Our theorem still follows, as long as we interpret $z(\mathbf{x,y})= 1$ as $\sigma(\mathbf{x,y})\ge 0.5$ and $z(\mathbf{x,y})= -1$ as $\sigma(\mathbf{x,y})< 0.5$. We can still convert the rectifier network into a system of linear inequalities even if the output unit is the sigmoid unit.}. Second, since $K$ ReLUs can represent $2^K-1$ constraints, the rectifier network can express constraints more compactly with fewer hidden units.

We will train the parameters $\mathbf{w}_k$'s and $b_k$'s of the rectifier network in the supervised setting. First, we need to obtain positive and negative training examples. We assume that we have training data for a structured prediction task.

\textbf{Positive examples} can be directly obtained from the training data of the structured prediction problem. For each training example $(\mathbf{x,y})$, we can apply constraint feature extractors to obtain positive examples of the form $(\psi(\mathbf{x,y}), +1)$.

\textbf{Negative examples} can be generated in several ways; we use simple but effective approaches. We can slightly perturb a structure $\mathbf{y}$ in a training example $(\mathbf{x,y})$ to obtain a structure $\mathbf{y}'$ that we assume to be invalid. Applying the constraint feature extractor to it gives a negative example $(\psi(\mathbf{x,y'}), -1)$.  We also need to ensure that $\psi(\mathbf{x,y'})$ is indeed different from any positive example. Another approach is to perturb the feature vector $\psi(\mathbf{x,y})$ directly, instead of perturbing the structure $\mathbf{y}$.

In our experiments in the subsequent sections, we will use both methods to generate negative
examples, with detailed descriptions in the supplementary material. Despite their simplicity, we observed performance improvements. Exploring more sophisticated methods for perturbing structures or features (e.g., using techniques explored by \citet{Smith2005}, or using adversarial learning~\citep{goodfellow2014generative}) is a future research direction.

To verify whether constraints can be learned as described here, we performed a synthetic experiment where 
we randomly generate many integer linear program (ILP) instances with \emph{hidden} shared constraints. The experiments show that constraints can indeed be recovered using only the solutions of the programs. Due to space constraints, details of this synthetic experiment are in the supplementary material. In the remainder of the paper we focus on three real NLP tasks.

\section{Entity and Relation Extraction Experiments}
In the task of entity and relation extraction, we are given a sentence with entity candidates. We seek to determine the type of each candidate, as in the following example (the labels are underlined):
\begin{quote}
[\underline{\lbl{Organization}} Google LLC] is headquartered in [\underline{\lbl{Location}} Mountain View, California].
\end{quote}

We also want to determine directed relations between the entities. In the above example, the relation from ``Google LLC'' to ``Mountain View, California'' is \lbl{OrgBasedIn}, and the opposite direction is labeled \lbl{NoRel}, indicating there is no relation.
This task requires predicting a directed graph and represents a typical structured prediction problem---we cannot make isolated entity and relation predictions.

\textbf{Dataset and baseline}:
We use the dataset from \cite{Roth2004}. It contains 1441 sentences with labeled entities and relations. There are three possible entity types: \lbl{Person}, \lbl{Location} and \lbl{Organization}, and five possible relations: \lbl{Kill}, \lbl{LiveIn}, \lbl{WorkFor}, \lbl{LocatedAt} and \lbl{OrgBasedIn}. Additionally, there is a special entity label \lbl{NoEnt} meaning a text span is not an entity, and a special relation label \lbl{NoRel} indicating that two spans are unrelated.

We used 70\% of the data for training  and the remaining 30\% for evaluation. We trained our baseline model using the integer linear program (ILP) formulation with the same set of features as in \cite{Roth2004}. The baseline system includes manually designed constraints from the original paper. An example of such a  constraint is: if a relation label is \lbl{WorkFor}, the source entity must be labeled \lbl{Person}, and the target entity must be labeled \lbl{Organization}. For reference, the supplementary material lists the complete set of manually designed constraints.

We use three kinds of constraint features: (i) source-relation indicator, which looks at a given relation label and the label of its source entity; (ii) relation-target indicator, which looks at a relation label and the label of its target entity; and (iii) relation-relation indicator, which looks at a pair of entities and focuses on the two relation label, one in each direction. The details of the constraint features, negative examples and hyper-parameters are in the supplementary material.

\subsection{Experiments and Results}
We compared the performance of two ILP-based models, both trained in the presence of constraints with a structured SVM. One model was trained with manually designed constraints and the other used learned constraints.  These models are compared in Table~\ref{er}.

\begin{table}[htbp]
\centering
\begin{tabular}{lrr}
\toprule
\bf Performance Metric & \bf Designed & \bf Learned \\ 
\midrule
entity F-1             & $84.1\%$        & $83.1\%$       \\
relation F-1           & $41.5\%$        & $38.2\%$       \\
\midrule
\end{tabular}
\caption{Comparison of performance on the entity and relation extraction task, between two ILP models, one trained with designed constraints (\textbf{Designed}) and one with learned constraints (\textbf{Learned}).}
\label{er}
\end{table}

We manually inspected the learned constraints and discovered that they exactly recover the designed constraints, in the sense that the feasible output space is exactly the same regardless of whether we use  designed or learned constraints.  As an additional confirmation, we observed that when a model is trained with designed constraints and tested with learned constraints, we get the same model performance as when tested with designed constraints. Likewise, a model that is trained with learned constraints performs identically when tested with learned and designed constraints. 

Below, we give one example of a learned constraint, and illustrate how to interpret such a constraint. (The complete list of learned constraints is in the supplementary material.) A learned constraint using the source-relation indicator features is

\begin{equation}
\begin{split}
& -1.98 x_1 + 3.53 x_2-1.90 x_3 +0.11 x_4 \\
&+ 2.66x_5-2.84x_6-2.84x_7-2.84x_8 \\
&+ 2.58x_9 + 0.43 x_{10} + 0.32 \ge 0
\end{split}
\label{example_ineq}
\end{equation}
where $x_1$ through $x_{10}$ are indicators for labels \lbl{NoEnt}, \lbl{Person}, \lbl{Location}, \lbl{Organization}, \lbl{NoRel}, \lbl{Kill}, \lbl{LiveIn}, \lbl{WorkFor}, \lbl{LocatedAt}, and \lbl{OrgBasedIn}, respectively. This constraint disallows a relation labeled as \lbl{Kill} having a source entity labeled as \lbl{Location}, because $-1.90-2.84+0.32<0$. Therefore, the constraint ``\lbl{Location} cannot \lbl{Kill}'' is captured in \eqref{example_ineq}. In fact, it is straightforward to verify that the inequality in~\eqref{example_ineq} captures many more constraints such as ``\lbl{NoEnt} cannot \lbl{LiveIn}'', ``\lbl{Location} cannot \lbl{LiveIn}'', ``\lbl{Organization} cannot \lbl{WorkFor}'', etc. A general method for interpreting learned constraints is a direction of future research.

Note that the metric numbers in Table~\ref{er} based on learned constraints are lower than those based on designed constraints. Since the feasible space is the same for both kinds of constraints, the performance difference is due to the randomness of the ILP solver picking different solutions with the same objective value. Therefore, the entity and relation experiments in this section demonstrate that our approach can recover the designed constraints and provide a way of interpreting these constraints.

\section{Citation Field Extraction Experiments}

In the citation field extraction task, the input is a citation entry. The goal is to identify spans corresponding to fields such as author, title, etc. In the example below, the labels are underlined:
\begin{quote}
\noindent
 [ \underline{\lbl{Author}} A . M . Turing . ]
 [ \underline{\lbl{Title}} Computing machinery and intelligence . ] 
 [ \underline{\lbl{Journal}} Mind , ] 
 [\underline{\lbl{Volume}} 59 , ] 
 [ \underline{\lbl{Pages}} 433-460 . ] 
 [ \underline{\lbl{Date}} October , 1950 . ]
\end{quote}
\citet{Chang2007} showed that \emph{hand-crafted} constraints specific to this domain can vastly help models to correctly identify citation fields. We show that constraints learned from the training data can improve a trained model without the need for manual effort.

\textbf{Dataset and baseline.} We use the dataset from \citet{Chang2007,Chang2012} whose training, development and test splits have 300, 100 and 100 examples, respectively. We train a first-order Markov model using structured SVM~\citep{Tsochantaridis2004} on the training set with the same raw text features as in the original work.


\paragraph{Constraint features.}
We explore multiple simple constraint features $\boldsymbol{\psi}(\mathbf{x, y})$ in the citation field extraction experiments as shown in~Table~\ref{tab:constraint-features}. Detailed descriptions of these features, including how to develop negative examples for each feature, and experiment settings are in the supplementary material.
\begin{table}[htbp]
\centering
\begin{tabular}{lp{0.23\textwidth}}
\toprule
\bf Feature & \bf Description \\ 
\midrule
Label existence & Indicates which labels exist in a citation       \\ 
Label counts   & Counts the number of occurrences of a label       \\
Bigram labels   & Indicators for adjacent labels  \\
Trigram labels & Indicators for $3$ adjacent labels \\
Part-of-speech  & Indicator for the part-of-speech of a token \\
Punctuation    & Indicator for whether a token is a punctuation \\
\bottomrule
\end{tabular}
\caption{Constraint feature templates for the citation field extraction task}
\label{tab:constraint-features}
\end{table}

\subsection{Experiments and Results}

For each constraint feature template, we trained a rectifier network with 10 ReLUs in the hidden layer. We then use Theorem~\ref{theorem} to convert the resulting network to a system of $2^{10}-1$, or 1023 linear inequalities. We used beam search with beam size 50 to combine the learned inequalities with the original sequence model to predict on the test set. States in the search space correspond to partial assignments to a prefix of the sequence. Each step we predict the label for the next token in the sequence. The pretrained sequence model (i.e., the baseline) ranks search nodes based on transition and emission scores, and the learned inequality prunes the search space accordingly\footnote{Since the label-existence and label-counts features are global, pruning by learned inequalities is possible only at the last step of search. The other four features admit pruning at each step of the search process.}. Table~\ref{citation} shows the token level accuracies of various methods.

\begin{table*}[htbp]
\centering
\begin{tabular}{lllllllllll}
\toprule
\multicolumn{2}{c}{\bf Baselines} &\multicolumn{6}{c}{\bf Search with learned constraints} &\multicolumn{3}{c} {\bf Combine constraints}\\
\cmidrule(r){1-2}\cmidrule(rl){3-8} \cmidrule(l){9-11}
  \bf Exact & \bf Search & \bf L.E. & \bf L.C. & \bf B.L. & \bf T.L. & \bf POS & \bf Punc. & \bf C1 & \bf C2 & \bf C3 \\ 
  86.2  & 87.3     & 88.0& 87.7 & 87.9  & 88.1 &89.8 &90.2 & 88.6& 90.1 & 90.6         \\ 
\midrule
\end{tabular}
\caption{Token level accuracies (in percentage) of baseline models and constrained-search models, for the citation field extraction task. \textbf{Exact} is our trained first-order Markov model. It uses exact inference (dynamic programming) for prediction. \textbf{Search} is our search baseline, it uses the same model as \textbf{Exact}, but with beam search for inexact inference. \textbf{L.E., L.C., B.L., T.L., POS, Punc.} use search with different constraint features: label existence, label counts, bigram labels, trigram labels, part-of-speech, and punctuation features. \textbf{C1} to \textbf{C3} are search with combined constraints. \textbf{C1} combines \textbf{L.E.} and \textbf{T.L.}. \textbf{C2} combines \textbf{L.E.}, \textbf{T.L.} and \textbf{POS}. Finally \textbf{C3} combines all constraints.}
\label{citation}
\end{table*}

The results show that all versions of constrained search outperform the baselines, indicating that the learned constraints are effective in the citation field extraction task. Furthermore, different constraints learned with different features can be combined. We observe that combining different constraint features generally improves accuracy.

It is worth pointing out that the label existence and label counts features are global in nature and cannot be directly used to train a sequence model. Even if some constraint features can be used in training the original model, it is still beneficial to learn constraints from them. For example, the bigram label feature is captured in the original first order model, but adding constraints learned from them still improves performance. As another test, we trained a model with POS features, which also contains punctuation information. This model achieves 91.8\% accuracy. Adding constraints learned with POS improves the accuracy to 92.6\%; adding constraints learned with punctuation features further improves it to 93.8\%.

We also observed that our method for learning constraints is robust to the choice of the number of hidden ReLUs. For example, for punctuation, learning using 5, 8 and 10 hidden ReLUs results an accuracy of $90.1\%$, $90.3\%$, and $90.2\%$, respectively. We observed similar behavior for other constraint features as well. Since the number of constraints learned is exponential in the number of hidden units, these results shows that learning redundant constraints will not hurt performance.


Note that carefully hand-crafted constraints may achieve higher accuracy than the learned ones. \citet{Chang2007} report an accuracy of 92.5\% with constraints specifically designed for this domain. In contrast, our method for learning constraints uses general constraint features, and does not rely on domain knowledge. Therefore, our method is suited to tasks where little is known about the underlying domain.

\section{Chunking Experiments}
Chunking is the task of clustering text into groups of syntactically correlated tokens or phrases. In the instance below, the phrase labels are underlined: 
\begin{quote}
[\underline{\lbl{NP}} An A.P. Green official] [\underline{\lbl{VP}} declined to comment] [\underline{\lbl{PP}} on] [\underline{\lbl{NP}} the filing] [\underline{\lbl{O}}.]
\end{quote}
We treat the chunking problem as a sequence labeling problem by using the popular IOB tagging scheme.
For each phrase label, the first token in the phrase is labeled with a ``B-'' prefixed to phrase label while the other tokens are labeled with an ``I-'' prefixed to the phrase label. Hence, 
\begin{quote}
[\underline{\lbl{NP}} An A.P. Green official]
\end{quote} is represented as
\begin{quote}
    [[\underline{\lbl{B-NP}} An] [\underline{\lbl{I-NP}} A.P.] [\underline{\lbl{I-NP}} Green] [\underline{\lbl{I-NP}} official]]
\end{quote} 
This is done for all phrase labels except ``O".

\textbf{Dataset and Baselines.} We use the CoNLL2000 dataset \cite{tjong-kim-sang-buchholz-2000-introduction} which contains 8936 training sentences and 2012 test sentences. For our experiments, we consider 8000 sentences out of 8936 training sentences as our training set and the remaining 936 sentences as our development set. Chunking is a well-studied problem and showing performance improvements on full training dataset is difficult. However, we use this task to illustrate the interplay of learned constraints with neural network models, and the impact of learned constraints in the low training data regime. 

We use the BiLSTM-CRF \cite{huang2015bidirectional} for this sequence tagging task. We use GloVe for word embeddings. We do not use the BERT~\cite{devlin2019bert} family of models since tokens are broken down into sub-words during pre-processing, which introduces modeling and evaluation choices that are orthogonal to our study of label dependencies. As with the citation task, all our constrained models use beam search, and we compare our results to both exact decoding and beam search baselines. We use two kinds of constraint features:  (i) $n$-gram label existence, and (ii) $n$-gram part of speech. Details of the constraint features and construction of negative samples are given in the supplementary material.


\subsection{Experiments and Results}
We train the rectifier network with 10 hidden units. The beam size of 10 was chosen for our experiments based on preliminary experiments. We report the average results on two different random seeds for learning each constraint. Note that the $n$-gram label existence is a global constraint while the $n$-gram POS constraint is a local constraint which checks for validity of label assignments at each token. In essence, the latter constraint reranks the beam at each step by ensuring that states that satisfy the constraint are preferred over states that violate the constraint. 
Since the $n$-gram label existence is a global constraint, we check the validity of the tag assignments only at the last token. In the case where none of the states in the beam satisfy the constraint, the original beams are used. 

\begin{table*}[!htbp]
\centering
\begin{tabular}{@{}cccccccc@{}}
\toprule
\multirow{2}{*}{\textbf{Constraint}}                      & \multirow{2}{*}{$n$}                 & \multicolumn{6}{c}{\textbf{Percentage of training data used}} \\ \cmidrule(l){3-8} 
                                                          &                                             & \textbf{1\%} & \textbf{5\%} & \textbf{10\%} & \textbf{25\%} & \textbf{50\%} & \textbf{100\%} \\ \midrule
\multirow{2}{*}{\textbf{Label existence}}                 & \textbf{2}                                  & 81.28       & 88.30       & 89.73      & 91.24       & 90.40      & 92.48    \\
                                                          & \textbf{3}                                  & 80.98       & 88.20       & 90.58      & 91.20       & 92.37      & 93.12    \\ \midrule
\multirow{2}{*}{\textbf{Part-of-speech}}                  & \textbf{3}                                  & 86.52       & 90.74       & 91.80      & 92.41       & 93.07      & 93.84    \\
                                                          & \textbf{4}                                  & 84.21       & 90.99       & 92.17      & 92.46       & 93.08      & 93.93    \\ \midrule
\multicolumn{2}{c}{\textbf{\begin{tabular}[c]{@{}l@{}}Search without constraints\end{tabular}}} & 81.29       & 88.27       & 90.62      & 91.33       & 92.51      & 93.44    \\ \midrule
\multicolumn{2}{
c}{\textbf{Exact decoding}}                                                           & 82.11       & 88.70       & 90.49      & 92.57       & 93.94      & 94.75    \\ \bottomrule
\end{tabular}
\caption{Token level accuracies (in percentage) for the chunking  baseline and constrained model. The results are shown on $n$-gram Label Existence and $n$-gram Part of Speech constraints with $n=\{2,3\}$ and $n= \{3,4\}$ respectively. The results are shown on $\{1\%, 5\%, 10\%, 25\%, 50\%, 100\%\}$  of training data. Exact decoding with Viterbi algorithm and Search w/o constraint are baseline models which do not incorporate constraints during inference. }
\label{chunking-exp}
\end{table*}

The results for this set of experiments are presented in Table \ref{chunking-exp}. We observe that the POS constraint improves the performance of the baseline models significantly, outperforming the beam search baseline on all training ratios. More importantly, the results show sizable improvements in accuracy for smaller training ratios (e.g, $4.41\%$ and $5.23\%$ improvements on exact and search baselines respectively with 1\% training data ). When the training ratios get bigger, we expect the models to learn these properties and hence the impact of the constraints decreases.

These results (along with the experiments in the previous sections) indicate that our constraints can significantly boost performance in the low data regime. Another way to improve performance in low resource settings is to use better pretrained input representations. When we replaced GloVe embeddings with ELMo,
we observed a $87.09\%$ accuracy on 0.01 ratio of training data using exact decoding. However, this improvement comes at a cost: the number of parameters increases from 3M (190k trainable) to 94M (561k trainable). In contrast, our method instead introduces a smaller rectifier network with $\approx1000$ additional parameters while still producing similar improvements. In other words, using trained constraints is computationally more efficient. 

We observe that the label existence constraints, however, do not help. We conjecture that this may be due to one of the following three conditions: (i) The label existence constraint might not exist for the task; (ii) The constraint exists but the learner is not able to find it; (iii) The input representations are expressive enough to represent the constraints. Disentangling these three factors is a future research challenge.

\section{Related Work}

\emph{Structured prediction} is an active field in machine learning and has numerous applications, including various kinds of sequence labeling tasks, parsing~\cite[e.g.,][]{Martins2009}, image segmentation~\cite[e.g.,][]{lam2015hc}, and information extraction~\cite[e.g.,][]{Anzaroot2014}. The work of \citet{Roth2004} introduced the idea of using explicitly stated constraints in an integer programming framework. That constraints and knowledge can improve models has been highlighted by several lines of work~\cite[e.g.,][]{Ganchev2010,Chang2012,hu2016harnessing}.

The interplay between constraints and representations has been sharply
highlighted by recent work on integrating neural networks with
structured outputs~\citep[e.g.,][and
others]{rocktaschel2017end,niculae2018sparsemap,manhaeve2018deepproblog,pmlr-v80-xu18h,li-srikumar-2019-augmenting,li-etal-2019-logic}. We
expect that constraints learned as described in this work can be integrated into these
formalisms, presenting an avenue for future research.

While our paper focuses on learning explicit constraints directly from examples, it is also possible to use indirect supervision from these examples to learn a structural classifier~\citep{Chang2010}, with an objective function penalizing invalid structures.

Related to our goal of learning constraints is \emph{rule learning}, as studied in various subfields of artificial intelligence. \citet{Quinlan1986} describes the ID3 algorithm, which extracts rules as a decision tree. First order logic rules can be learned from examples using inductive logic programming \citep{MUGGLETON1994629,Lavrac1994,Page2003}. Notable algorithms for inductive logic programming include FOIL \citep{Quinlan1990} and Progol \citep{Muggleton1995}.

\emph{Statistical relation learning} addresses learning constraints with uncertainty~\citep{Friedman1999,Getoor2001}. Markov logic networks~\citep{Richardson2006} combines probabilistic models with first order logic knowledge, whose weighted formulas are soft constraints and the weights can be learned from data. In contrast to these directions, in this paper, we exploit a novel representational result about rectifier networks to learn polytopes that represent constraints with off-the-shelf neural network tools.

\section{Conclusions}

We presented a systematic way for discovering constraints as linear inequalities for structured prediction problems. The proposed approach is built upon a novel transformation from two layer rectifier networks to linear inequality constraints and does not rely on domain expertise for any specific problem. Instead, it only uses general constraint features as inputs to rectifier networks. Our approach is particularly suited to tasks where designing constraints manually is hard, and/or the number of training examples is small. The learned constraints can be used for structured prediction problems in two ways: (1) combining them with an existing model to improve prediction performance, or (2) incorporating them into the training process to train a better model. We demonstrated the effectiveness of our approach on three NLP tasks, each with different original models.

\section*{Acknowledgments}

We thank members of the NLP group at the University of Utah, especially Jie Cao,
for their valuable insights and suggestions; and reviewers for pointers to
related works, corrections, and helpful comments. We also acknowledge the
support of NSF Cyberlearning-1822877, SaTC-1801446 and gifts from Google and
NVIDIA.

\bibliographystyle{acl_natbib}
\bibliography{main}

\appendix

\section{Proof of Theorem 1}
In this section we prove Theorem 1. The theorem and the relevant definitions are repeated here for convenience.

Define the rectifier (ReLU) activation function as $\R(x)= \max(0, x)$. Consider the following two-layer rectifier network:
\begin{equation}\label{relu_network2}
z(\mathbf{x,y}) = \sgn\Big(1- \sum_{k=1}^K \R \big(\mathbf{w}_k \cdot \boldsymbol{\psi}(\mathbf{x, y}) + b_k\big)\Big)
\end{equation}
The input to the network is still $\boldsymbol{\psi}(\mathbf{x, y})$. There are $K$ ReLUs in the hidden layer, and one threshold unit in the output layer.

The decision boundary of this rectifier network is specified by a system of linear inequalities. In particular, we have the following theorem:
\begin{theorem}\label{theorem2}
Consider a two-layer rectifier network with $K$ hidden ReLUs as in Eq.~\eqref{relu_network2}. Define the set $[K]=\{1,2,\dots,K\}$. The network outputs $z(\mathbf{x,y})=1$ if, and only if, for every subset $\mathcal{S}$ of $[K]$, the following linear inequality holds:
\begin{equation*}
1-\sum_{k \in \mathcal{S}}\big(\mathbf{w}_k \cdot \boldsymbol{\psi}(\mathbf{x, y}) + b_k\big) \ge 0
\end{equation*}
\end{theorem}
\begin{proof}
Define $a_k=\mathbf{w}_k \cdot \boldsymbol{\psi}(\mathbf{x, y}) + b_k$. We first prove the ``if'' part of the theorem. Suppose that for any $\sS\subseteq [K]$, $1-\sum_{k\in\sS}a_k \ge 0$. Thus for a specific subset $\sS^*= \{k \in [K]: a_k \ge 0\}$, we have $1-\sum_{k\in\sS^*}a_k \ge 0$.
By the definition of $\sS^*$, $\sum_{k=1}^K \rect(a_k) = \sum_{k \in \sS*} a_k$, therefore $1-\sum_{k=1}^K \rect(a_k) \ge 0$. 

Next we prove the ``only if'' part of the theorem. Suppose that $1-\sum_{k=1}^K \rect(a_k) \ge 0$. For any $\sS\subseteq [K]$, we have $\sum_{k=1}^K \rect(a_k) \ge \sum_{k\in \sS} \rect(a_k) \ge \sum_{k\in \sS} a_k$. Therefore, for any $\sS\subseteq [K]$, $1-\sum_{k\in\sS}a_k \ge 0$.
\end{proof}

\section{Synthetic Integer Linear Programming Experiments}

We first check if constraints are learnable, and whether learned constraints help a downstream task with a synthetic experiment.
Consider framing structure prediction as an integer linear program (ILP):
\begin{equation}
\begin{split}
\min_{\bz \in \{0, 1\}^n} \sum_{i} c_i \cdot z_i, \quad\quad\quad \\\\ \text{subject to} \quad \sum_i A_{ki}z_i \ge b_k,    \quad k\in[m]
\end{split}
\label{ilp-constraints}
\end{equation}

The objective coefficient $c_i$ denotes the cost of setting the variable $z_i$ to $1$ and the goal of prediction is to find a cost minimizing variable assignment subject to $m$ linear constraints in \eqref{ilp-constraints}.
We randomly generate a hundred 50-dimensional ILP instances, all of which share a fixed set of random constraints. Each instance is thus defined by its objective coefficients. We reserve $30\%$ of instances as test data. The goal is to learn the shared linear constraints in Eq.~\eqref{ilp-constraints} from the training set.

We use the Gurobi Optimizer~\citep{gurobi} to solve all the ILP instances to obtain pairs  $\{\mathbf{(c, z)}\}$, where $\mathbf{c}$ is the vector of objective coefficients and $\mathbf{z}$ is the optimal solution. 
Each $\mathbf{z}$  in this set is feasible, giving us positive examples $(\mathbf{z}, +1)$ for the constraint learning task.

Negative examples are generated as follows: Given a positive pair $\mathbf{(c, z)}$ described above, if the $i^\text{th}$ coefficient $c_i > 0$ and the corresponding decision $z_i = 1$, construct $\mathbf{z'}$ from $\mathbf{z}$ by flipping the $i^\text{th}$ bit in $\mathbf{z}$ from $1$ to $0$. Such a $\mathbf{z'}$ is a negative example for the constraint learning task because $\mathbf{z'}$ has a lower objective value than $\mathbf{z}$. Therefore, it violates at least one of the constraints in Eq.~\eqref{ilp-constraints}. Similarly, if $c_i < 0$ and $z_i = 0$, we can flip the $i^\text{th}$ bit from $0$ to $1$. We perform the above steps for every coefficient of every example in the training set to generate a set of negative examples $\{(\mathbf{z'}, -1)\}$.

We trained a rectifier network on these examples and converted the resulting parameters into a system of linear inequalities using Theorem~\ref{theorem2}. The hyper-parameters and design choices are summarized in the supplementary material. We used the learned inequalities to replace the original constraints to obtain predicted solutions. We evaluated these predicted solutions against the oracle solutions (i.e., based on the original constraints). We also computed a baseline solution for each test example by minimizing an unconstrained objective.

Table~\ref{table:synthetic-ilp} lists four measures of the effectiveness of learned constraints. First, we want to know whether the learned rectifier network can correctly predict the synthetically generated positive and negative examples. The binary classification accuracies are listed in the first row. The second row lists the bitwise accuracies of the predicted solutions based on learned constraints, compared with the gold solutions. We see that the accuracy values of the solutions based on learned constraints are in the range from $80.2$--$83.5\%$. As a comparison, without using any constraints, the accuracy of the baseline is $56.8\%$. Therefore the learned constraints can substantially improve the prediction accuracy in the down stream inference tasks. The third row lists the percentage of the predicted solutions satisfying the original constraints. Solutions based on learned constraints satisfy $69.8$--$74.4\%$ of the original constraints. In contrast, the baseline solutions satisfy $55.3\%$ of the original constraints. The last row lists the percentage of the gold solutions satisfying the learned constraints. We see that the gold solutions almost always satisfy the learned constraints.

\begin{table*}[htbp]
   \centering
   \begin{tabular}{lccccccccc} 
      \toprule
      \multicolumn{1}{c}{} & \multicolumn{9}{c}{Number of ReLUs} \\
      \cmidrule(r){2-10} 
      & 2 & 3 & 4 & 5 & 6 & 7 & 8 & 9 & 10\\
      \midrule
     binary classification acc. ($\%$) & 85.1 & 87.3 & 92.1 & 90.3 & 95.0 & 94.3 & 94.1 & 97.7 & 98.0 \\
      bitwise solution acc. ($\%$) & 81.1 & 80.9 & 81.9 & 80.2 & 81.0 & 82.3 & 81.1 & 83.2 & 83.5  \\
      original constr. satisfied ($\%$)& 70.3 & 69.8 & 72.7 & 70.4 & 70.1 & 71.1 & 71.4 & 74.4 & 74.3 \\
      learned constr. satisfied ($\%$)& 95.6 & 98.6 & 98.7 & 99.1 & 97.4 & 98.9 & 99.9 & 99.1 & 99.4 \\
      \bottomrule
   \end{tabular}
   \caption{Effectiveness of learned constraints for the synthetic ILP experiments.}
   \label{table:synthetic-ilp}
\end{table*}

The hyper-parameter and other design choices for
the synthetic ILP experiments are listed in Table~\ref{tab:ilp-para}.

\begin{table*}[htbp]
   \centering
   \begin{tabular}{@{} lr @{}} 
      \toprule
      Description & Value \\
      \midrule
  Total number of examples                                              & 100                                                \\
  Number of training examples                                           & 70                                                 \\
  Number of test examples                                               & 30                                                 \\
  Dimensionality                                                        & 50                                                 \\
  Range of hidden ReLU units considered for experiments                 & 2-10                                               \\
  Learning rates for cross-validation while learning rectifier networks & $\{0.001, 0.01, 0.1\}$                             \\
  Learning rate decay for cross-validation                              & $\{0.0, 10^{-7}, 10^{-6}\}$                        \\
  Optimizer parameters for learning                                     & $\beta_1 = 0.9, \beta_2 = 0.999, \epsilon=10^{-7}$ \\
  Number of training epochs                                             & 1000                                               \\
      \bottomrule
   \end{tabular}
   \caption{Parameters used in the synthetic ILP experiments}
   \label{tab:ilp-para}
\end{table*}

\section{Entity and relation extraction experiments}

\subsection{Designed constraints}

Table~\ref{table:designed_constraints} lists the designed constraints used in the entity and relation extraction experiments. There are fifteen constraints, three for each relation type. For example, the last row in Table~\ref{table:designed_constraints} means that the relation \lbl{OrgBasedIn} must have an \lbl{Organization} as its source entity and a \lbl{Location} as its target entity, and the relation in the opposite direction must be \lbl{NoRel}.

\begin{table*}[htbp]
  \centering
  \begin{tabular}{llll}
    \toprule
    \multicolumn{1}{c}{Antecedents} & \multicolumn{3}{c}{Consequents}                   \\
    \cmidrule(r){1-1}\cmidrule(l){2-4} 
    If the relation is     & Source must be     & Target must be & Reversed relation must be \\
    \cmidrule(r){1-1}\cmidrule(l){2-4} 
    \lbl{Kill} & \lbl{Person}  & \lbl{Person} & \lbl{NoRel}     \\
    \lbl{LiveIn}     & \lbl{Person} & \lbl{Location} & \lbl{NoRel}     \\
    \lbl{WorkFor}    & \lbl{Person}       & \lbl{Organization} & \lbl{NoRel}  \\
    \lbl{LocatedAt} & \lbl{Location} & \lbl{Location} & \lbl{NoRel} \\
    \lbl{OrgBasedIn} & \lbl{Organization} & \lbl{Location} & \lbl{NoRel} \\
    \bottomrule
  \end{tabular}
    \caption{Designed constraints used in the entity and relation extraction experiments}
  \label{table:designed_constraints}
\end{table*}

\subsection{Constraint features}

We use the same example as in the main paper to illustrate the constraint features used in the entity and relation extraction experiments:

\begin{quote}
[\underline{\lbl{Organization}} Google LLC] is headquartered in [\underline{\lbl{Location}} Mountain View, California, USA].
\end{quote}

In the above example, the relation from ``Google LLC'' to ``Mountain View, California, USA'' is \lbl{OrgBasedIn}, and the relation in the opposite direction is labeled \lbl{NoRel}, indicating there is no relation from ``Mountain View, California, USA'' to ``Google LLC''.

We used three constraint features for this task, explained as follows.

\paragraph{Source-relation indicator}
This feature looks at a given relation label and the label of its source entity. It is an indicator pair (source label, relation label). Our example sentence will contribute the following two feature vectors, (\lbl{Organization}, \lbl{OrgBasedIn}) and (\lbl{Location}, \lbl{NoRel}), both corresponding to postive examples. The negative examples contains all possible pairs of (source label, relation label), which do not appear in the positive example set.

\paragraph{Relation-target indicator}
This feature looks at a given relation label the label of its target entity. It is an indicator pair (relation label, target label). Our example sentence will contribute the following two feature vectors, (\lbl{OrgBasedIn}, \lbl{Location}) and (\lbl{NoRel},\lbl{Organization}), both corresponding to positive examples. The negative examples contains all possible pairs of (relation label, target label), which do not appear in the positive example set.

\paragraph{Relation-relation indicator}
This feature looks at a pair of entities and focuses on the two relation labels between them, one in each direction. Therefore our running example will give us two positive examples with features (\lbl{OrgBasedIn}, \lbl{NoRel}) and (\lbl{NoRel},\lbl{OrgBasedIn}). The negative examples contain any pair of relation labels that is not seen in the positive example set.

\subsection{Hyper-parameters and design choices}
The hyper-parameter and design choices for the experiments are in Table~\ref{tab:entity-relation-para}. Note that different runs of the SVM learner with the learned or designed constraints may give different results from those on Table~\ref{er}. This is due to non-determinism introduced by hardware and different versions of the Gurobi solver picking different solutions that have the same objective value. In the results in Table~\ref{er}, we show the results where the training with learned constraints seem to underperform the model that is trained with designed constraints. In other runs on different hardware, we found the opposite ordering of the results.

\begin{table*}[htbp]
   \centering
   \begin{tabular}{@{} lr @{}} 
      \toprule
      Description & Value \\
      \midrule
Structured SVM trade-off parameter for the base model                 & $2^{-6}$                                           \\
  Number of hidden ReLU units                                           &                                                    \\
  \quad --for source-relation                                      & 2                                                  \\
  \quad --for relation-target                                      & 2                                                  \\
  \quad --for relation-relation                                    & 1                                                  \\
  Learning rates for cross-validation while learning rectifier networks & $\{0.001, 0.01, 0.1\}$                             \\
  Learning rate decay for cross-validation                              & $\{0.0, 10^{-7}, 10^{-6}\}$                        \\
  Optimizer parameters for learning                                     & $\beta_1 = 0.9, \beta_2 = 0.999, \epsilon=10^{-7}$ \\
      \bottomrule
   \end{tabular}
   \caption{Parameters used in the entity and relation extraction experiments}
   \label{tab:entity-relation-para}
\end{table*}

\subsection{Learned Constraints}

We see in the main paper that $2^K-1$ linear inequality constraints are learned using a rectifier network with $K$ hidden units. In the entity and relation extraction experiments, we use two hidden units to learn three constraints from the source-relation indicator features. The three learned constraints are listed in Table~\ref{table:source-relation-constraints}. A given pair of source label and relation label satisfies the constraint if the sum of the corresponding coefficients and the bias term is greater than or equal to zero. For example, the constraint from the first row in Table~\ref{table:source-relation-constraints} disallows the pair (\lbl{Location}, \lbl{Kill}), because $-1.90-2.84+0.32<0$. Therefore, the learned constraint would not allow the source entity of a \lbl{Kill} relation to be a \lbl{Location}, which agrees with the designed constraints.

\begin{table*}[htbp]
  \centering
  \begin{tabular}{rrrrrrrrrrr}
    \toprule
    \multicolumn{4}{c}{Source Labels}  &  \multicolumn{6}{c}{Relation Labels}         \\
    \cmidrule(r){1-4} \cmidrule(lr){5-10}
    \lbl{NoEnt}     & \lbl{Per.}     & \lbl{Loc.} & \lbl{Org.} & \lbl{NoRel} & \lbl{Kill} & \lbl{Live} & \lbl{Work} & \lbl{Located} & \lbl{Based}  & Bias      \\
    \cmidrule(r){1-4} \cmidrule(lr){5-10} \cmidrule(l){11-11}
    -1.98 & 3.53  & -1.90 & 0.11 & 2.66 & -2.84 & -2.84 & -2.84 & 2.58 & 0.43 & 0.32    \\
    -1.61 & -1.48 & 3.50 & 0.92 & 1.15 & 1.02 & 1.02 & 1.02 & -3.96 & -1.38 & 1.46   \\
    -3.59 & 2.04  & 1.60 & 1.03 & 3.81 & -1.82 & -1.82 & -1.82 & -1.38 & -0.95 & 0.78  \\
    \bottomrule
  \end{tabular}
  \caption{Linear constraint coefficients learned from the source-relation indicator features}
  \label{table:source-relation-constraints}
\end{table*}

We enumerated all possible pairs of source label and relation label and found that the learned constraints always agree with the designed constraints in the following sense: whenever a pair of source label and relation label satisfies the designed constraints, it also satisfies all three learned constraints, and whenever a pair of source label and relation label is disallowed by the designed constraints, it violates at least one of the learned constraints. Therefore, our method of constraint learning exactly recovers the designed constraints.
\begin{table*}[htbp]
  \centering
  \small
  \begin{tabular}{rrrrrrrrrrrrr}
    \toprule
    \multicolumn{6}{c}{Forward Relation Labels}  &  \multicolumn{6}{c}{Backward Relation Labels} & \multicolumn{1}{c}{Bias}              \\
    \cmidrule(r){1-6} \cmidrule(rl){7-12}\cmidrule(l){13-13}
    4.95 & -1.65 & -1.65 & -1.65 & -1.65 & -1.65 & 5.06 & -1.53 & -1.53& -1.53& -1.53& -1.53 & -2.41 \\
    \bottomrule
  \end{tabular}
    \caption{Linear constraint coefficients learned from the relation-relation indicator features. The order of the relation labels is: \lbl{NoRel}, \lbl{Kill}, \lbl{LiveIn}, \lbl{WorkFor}, \lbl{LocatedAt}, and \lbl{OrgBasedIn}}
  \label{table:relation-relation-constraint}
\end{table*}

We also use two hidden units to learn three constraints from the relation-target indicator features, and one hidden unit to learn one constraint from the relation-relation indicator features. The learned constraints are listed in Table~\ref{table:relation-target-constraints} and Table~\ref{table:relation-relation-constraint}. Again we verify that the learned constraints exactly recover the designed constraints in all cases.

\begin{table*}[htbp]
  \centering
  \begin{tabular}{rrrrrrrrrrr}
    \toprule
    \multicolumn{6}{c}{Relation Labels}  &  \multicolumn{4}{c}{Target Labels}               \\
    \cmidrule(r){1-6} \cmidrule(lr){7-10}
    \lbl{NoRel} & \lbl{Kill} & \lbl{Live} & \lbl{Work} & \lbl{Located} & \lbl{Based} &
    \lbl{NoEnt}     & \lbl{Per.}     & \lbl{Loc.} & \lbl{Org.} &  Bias \\
    \cmidrule(r){1-6} \cmidrule(lr){7-10}\cmidrule(l){11-11}
    2.68 & -3.17 & -0.55 & 2.68 & -0.55 & -0.55 & -1.58 & 3.15 & 0.53 & -2.70 & 1.02  \\
    2.72 & 2.42 & -1.39 & -2.55 & -1.39 & -1.39 & -2.51 & -2.27 & 1.54 & 2.31 & 0.85 \\
    5.40 & -0.74 & -1.94 & 0.13 & -1.94 & -1.94 & -4.10 & 0.88 & 2.08 & -0.39 & 0.86 \\
    \bottomrule
  \end{tabular}
    \caption{Linear constraint coefficients learned from the relation-target indicator features}
  \label{table:relation-target-constraints}
\end{table*}

\section{Citation field extraction experiments}

\subsection{Constraint Features}
\label{supp:citation_cons}
We use the same example as in the main paper to illustrate the constraint features used in the citation field extraction experiments:

\begin{quote}
  [ \underline{\lbl{Author}} A . M . Turing . ] [ \underline{\lbl{Title}} Computing machinery and intelligence . ] [ \underline{\lbl{Journal}} Mind , ] [\underline{\lbl{Volume}} 59 , ] [ \underline{\lbl{Pages}} 433-460 . ] [ \underline{\lbl{Date}} October , 1950 . ]
\end{quote}
We explore multiple simple constraint features $\boldsymbol{\psi}(\mathbf{x, y})$ as described below.

\paragraph{Label existence}
This features indicates which labels exist in a citation entry. In our above example, there are six labels. Suppose there are $n_l$ possible labels. The above example is a positive example, the feature vector of which is an $n_l$-dimensional binary vector. Exactly six elements, corresponding to the six labels in the example, have the value 1 and all others have the value 0.
To obtain the negative examples, we iterate through every positive example and flip one bit of its feature vector. If the resulting vector is not seen in the positive set it will be a negative example.

\paragraph{Label counts}
Label-count features are similar to Label-existence features. Instead of indicating whether a label exists using 1 or 0, label-count features records the number of times each label appears in the citation entry. The positive examples can be generated naturally from the training set.
To generate negative examples, we perturb the actual labels of a positive example, as opposed to its feature vector. We then extract the label counts feature from the perturbed example, and treat it as negative if it has not seen before in the positive set.

\paragraph{Bigram labels}
This feature considers each pair of adjacent labels in the text. From left to right, the above example will give us feature vectors like (\lbl{Author}, \lbl{Author}), (\lbl{Author}, \lbl{Title}), (\lbl{Title}, \lbl{Title}), \dots, (\lbl{Date}, \lbl{Date}). We then use one-hot encoding to represent these features, which is the input vector to the rectifier network. All these feature vectors are labeled as positve (+1) by the rectifier network, since they are generated from the training set. 
To generate negative examples for bigram-label features, we generate all positive examples from the training set, then enumerate all possible pair of labels and select those that were not seen in the positive examples.

\paragraph{Trigram labels}
This feature is similar to the bigram labels. From the training set, we generate positive examples, e.g., (\lbl{Author}, \lbl{Author}, \lbl{Author}), (\lbl{Author}, \lbl{Author}, \lbl{Title}) etc, and convert them into one-hot encodings. 
For negative examples, we enumerate all possible trigram labels, and select those trigrams as negative if two conditions are met: (a) the trigram is not seen in the positive set; and (b) a bigram contained in it is seen in the training set. The intuition is that we want negative examples to be almost feasible.

\paragraph{Part-of-speech}
For a fixed window size, we extract part-of-speech tags and the corresponding labels, and use the combination as our constraint features. For example, with window size two, we get indicators for ($\lbl{tag}_{i-1}$, $\lbl{tag}_i$, $\lbl{label}_{i-1}$, $\lbl{label}_i$) for the $i^\text{th}$ token in the sentence, where $\lbl{tag}$ and $\lbl{label}$ refer to part-of-speech tag and citation field label respectively.
For negative examples, we enumerate all four-tuples as above, and select it as negative if the four-tuple is not seen in the positive set, but both  ($\lbl{tag}_{i-1}$, $\lbl{tag}_i$) and ($\lbl{label}_{i-1}$, $\lbl{label}_i$) are seen in the training set.

\paragraph{Punctuation}
The punctuation feature is similar to the part-of-speech feature. Instead of the POS tag, we use an indicator for whether the current token is a punctuation.

\subsection{Hyper-parameters and design choices}
The hyper-parameter and design choices for the experiments are in
the Table~\ref{tab:citation-para}.

\begin{table*}[htbp]
   \centering
   \begin{tabular}{@{} lr @{}} 
      \toprule
      Description & Value \\
      \midrule
  Structured SVM trade-off parameter for the base model                 & unregularized                                      \\
  Beam size                                                             & 50                                                 \\
  Number of hidden ReLU units for experiments                           & 10                                                 \\
  Learning rates for cross-validation while learning rectifier networks & $\{0.001, 0.01, 0.1\}$                             \\
  Learning rate decay for cross-validation                              & $\{0.0, 10^{-7}, 10^{-6}\}$                        \\
  Optimizer parameters for learning                                     & $\beta_1 = 0.9, \beta_2 = 0.999, \epsilon=10^{-7}$ \\
      \bottomrule
   \end{tabular}
   \caption{Parameters used in the citation field extraction experiments}
   \label{tab:citation-para}
\end{table*}

\section{Chunking Experiments}
\subsection{Constraint Features}
The two constraints which we discussed in the main paper for the chunking dataset are described below.
\paragraph{N-gram label existence}  This constraint is a general form of the label existence constraint mentioned in Section \ref{supp:citation_cons}. In fact, it is the n-gram label existence constraint with n=1. The n-gram label existence constraint represents the labels of a sequence as a binary vector. Each feature of this binary vector corresponds to an n-gram label combination. Hence, the length of this constraint feature will be $\mid l\mid^n$ where $\mid l\mid$ is the total number of distinct labels. This means the vector size of this constraint grows exponentially with increasing $n$. The binary vector indicates a value of 1 for all the n-gram label features present in the sequence tags. The positive examples are hence formed from the training set sequences. For the negative examples, we iterate through each positive example and flip a bit. The resulting vector is incorporated as a negative example if it doesn't occur in the training set. 
\paragraph{N-gram part of speech (POS)} This constraint is a general form of the part of speech constraint mentioned in Section \ref{supp:citation_cons}. POS tags of a token are converted to a indicator vector. We concatenate the indicator vectors of each gram in an n-gram in order and this vector is further concatenated with indicators of labels of each of these grams. Hence, for n=2, we get the constraint vector as   ($\lbl{tag}_{i-1}$, $\lbl{tag}_i$, $\lbl{label}_{i-1}$, $\lbl{label}_i$) where $\lbl{tag}_i$ and $\lbl{label}_i$ are indicators for POS tags and labels respectively for the $i^{th}$ token. The positive examples enumerate vectors for all existing n-grams in the training sequences. The negative examples are creating by changing a label indicator in the constraint feature. The label to be perturbed and the perturbation both are chosen at random. The constraint vector hence formed is incorporated as a negative example if it doesn't occur in the set of positive examples.

\subsection{Hyper-parameters and design choices}
The hyper-parameter and design choices are summarized in Table \ref{tab:chunk-para}.
\begin{table*}[htbp]
   \centering
   \begin{tabular}{@{} lr @{}} 
       \toprule
      Description & Value \\
      \midrule          
      Constraint Rectifier Network \\
      \midrule
  Range of hidden ReLU units considered for experiments                 & $\{5,10\}$                                               \\
  Learning rates for development while learning rectifier networks & $\{0.001, 0.005, 0.01, 10^{-4}\}$                             \\
  Number of training epochs                                             & 1000                                               \\
  Random Seeds                                                          & $\{1,2\}$ \\
  \midrule
  BiLSTM CRF Model\\
  \midrule
  Learning rate for development while learning baseline model & $\{0.01,0.05,0.001,0.005\}$ \\
  Learning Rate Decay                                         & $\{10^{-5},10^{-6}\}$ \\
  Beam Size                                                   & 10 \\
  Number of training epochs                                   & 300 \\
   \bottomrule
   \end{tabular}
   \caption{Parameters used in the chunking experiments}
   \label{tab:chunk-para}
\end{table*}


\end{document}